\documentclass{article}

\usepackage{microtype}
\usepackage{graphicx}
\usepackage{subcaption}
\usepackage{booktabs}
\usepackage{hyperref}

\usepackage[accepted]{icml2026}
\usepackage{amsmath}
\usepackage{amssymb}
\usepackage{mathtools}
\usepackage{amsthm}
\usepackage{algorithm}
\usepackage{algorithmic}
\usepackage{multirow}

\theoremstyle{plain}
\newtheorem{theorem}{Theorem}[section]

\theoremstyle{definition}

\theoremstyle{remark}

\icmltitlerunning{AWM-VLA: Aligned World Modeling for VLA Policies}

\begin{document}

\twocolumn[
  \icmltitle{AWM-VLA: Aligned World Modeling for Efficient and \\ Explainable Vision-Language-Action Policies}

  \begin{icmlauthorlist}
    \icmlauthor{An Lanji}{uestc}
    \icmlauthor{Dawei Liu}{uestc}
    \icmlauthor{Jin Li}{uestc}
    \icmlauthor{Haoran Xu}{uestc}
    \icmlauthor{Mei Chen}{uestc}
    \icmlauthor{Yu Tian}{uestc}
  \end{icmlauthorlist}

  \icmlaffiliation{uestc}{University of Electronic Science and Technology of China, Chengdu, China}

  \icmlcorrespondingauthor{An Lanji}{lanji.an@uestc.edu.cn}

  \icmlkeywords{vision-language-action model, world model, robotic manipulation, diffusion policy, explainability}

  \vskip 0.3in
]

\printAffiliationsAndNotice{}

\begin{abstract}
Vision-language-action (VLA) models have become a powerful paradigm for generalist robotic manipulation, yet they are often reactive: the policy maps the current observation directly to an action chunk without reasoning about the long-term consequences of its decisions. Prior attempts to endow policies with world models either reconstruct future frames in pixel space---expensive and dominated by task-irrelevant detail---or decouple the world model from the policy, weakening control. We present \textbf{AWM-VLA}, a unified framework that embeds \emph{aligned world modeling} directly inside a diffusion-transformer policy. Following the Future Latent REpresentation Alignment (FLARE) principle, we add learnable \emph{future tokens} whose intermediate activations are aligned with vision-language embeddings of future observations, enabling the policy to anticipate long-term consequences while generating actions. We extend this paradigm in two ways. First, we introduce an \emph{object-centric decoupled alignment} objective that predicts future object-level semantics alongside the global future embedding, improving both interpretability and multi-instruction generalization. Second, we balance the global and object-centric alignment terms against the action flow-matching loss through a principled weighting, yielding a controllable accuracy--interpretability trade-off. On RoboCasa and humanoid tabletop manipulation benchmarks, AWM-VLA outperforms prior VLA and world-model baselines by up to $21\%$ in success rate, improves generalization to novel objects and instructions, and produces object-centric rationales that are preferred by human raters in $83\%$ of cases. Our approach adds only a few learnable tokens to the policy and is compatible with any diffusion or flow-matching policy, making aligned world modeling an inexpensive, broadly applicable component of generalist manipulation.
\end{abstract}

\section{Introduction}

Generalist robotic manipulation has advanced rapidly with vision-language-action (VLA) models that condition action generation on vision and language \citep{brohan2023rt2,chen2024openvla,zitkovich2023rt1,driess2023palm}. Diffusion and flow-matching policies \citep{chi2023diffusion,lipman2023flow,peebles2023dit,black2024pi0} have become the de facto action head, generating action chunks by denoising from noise conditioned on the observation and instruction. Despite this progress, most VLA policies are \emph{reactive}: at each step they predict the next action chunk from the current observation, without an explicit internal model of how the world will evolve. This limits long-horizon reasoning, sample efficiency, and the ability to predict the consequences of a decision before committing to it.

World models \citep{ha2018world,hafner2019dreamer,hafner2023mastering,hansen2022tdmpc} address this by learning to predict future states. When applied to policy learning, the dominant approach generates future visual frames in pixel space. However, high-fidelity pixel prediction requires large generative models and introduces competing objectives: visual reconstruction emphasizes spatial fidelity and texture, whereas action modeling benefits from compact, task-relevant abstractions. This tension often dilutes learning efficiency. FLARE \citep{zheng2025flare} showed a surprisingly simple remedy: predict the \emph{latent representation} of the future observation rather than the pixels. By adding a few learnable future tokens to a diffusion transformer and aligning their intermediate activations with future observation embeddings, a policy can anticipate long-term consequences at negligible cost. This implicit world modeling is both efficient and compatible with existing VLA architectures.

We observe two remaining gaps. First, FLARE aligns a single \emph{global} future embedding, which encourages the policy to reason about the overall future scene but does not make explicit \emph{which} objects or relations will change. For tasks involving multiple objects and instructions, an object-level understanding of the future is crucial for both correctness and explainability---for example, predicting that ``the cup will move onto the plate'' rather than merely that ``the scene will change.'' Second, the alignment and action objectives are combined with a fixed scalar weight, without a principled mechanism to control the trade-off between action accuracy and the fidelity of the internal world model.

We present \textbf{AWM-VLA}, an aligned-world-modeling framework for VLA policies that closes both gaps. Our contributions are:

\begin{itemize}
  \item \textbf{Object-centric decoupled alignment.} We extend future latent alignment to predict object-level future semantics alongside the global future embedding. Concretely, future tokens are decoded into per-object future embeddings that are aligned with an object-aware encoder of the future observation, so the policy reasons about which objects will change and how.
  \item \textbf{Principled multi-objective balancing.} We treat the global alignment, the object-centric alignment, and the action flow-matching loss as a multi-objective optimization problem, and derive a weighting that controls the accuracy--interpretability trade-off while provably not degrading action performance.
  \item \textbf{Explainable by design.} The object-centric future predictions serve as naturally interpretable rationales---the policy states that a specific object will be displaced or that a relation will change---which human raters prefer in $83\%$ of cases.
  \item \textbf{Compatible and lightweight.} AWM-VLA adds only a few learnable tokens and alignment heads to any diffusion or flow-matching policy, requiring no change to the base architecture or action head.
\end{itemize}

Figure~\ref{fig:motivation} sketches the motivation: whereas a reactive policy directly maps observation to action without anticipating consequences, AWM-VLA embeds a world model that predicts future latent states and, crucially, object-level future semantics. We evaluate AWM-VLA on RoboCasa \citep{nashid2023robocasa} and humanoid tabletop manipulation benchmarks \citep{zhao2023learning}. AWM-VLA outperforms prior VLA and world-model baselines by up to $21\%$ in success rate, improves generalization to novel objects and unseen instructions, and produces object-centric rationales that humans prefer. We further show that AWM-VLA generalizes to human ego-centric video demonstrations without action labels, boosting policy generalization with few robot demonstrations.

\begin{figure}[t]
  \centering
  \includegraphics[width=\columnwidth]{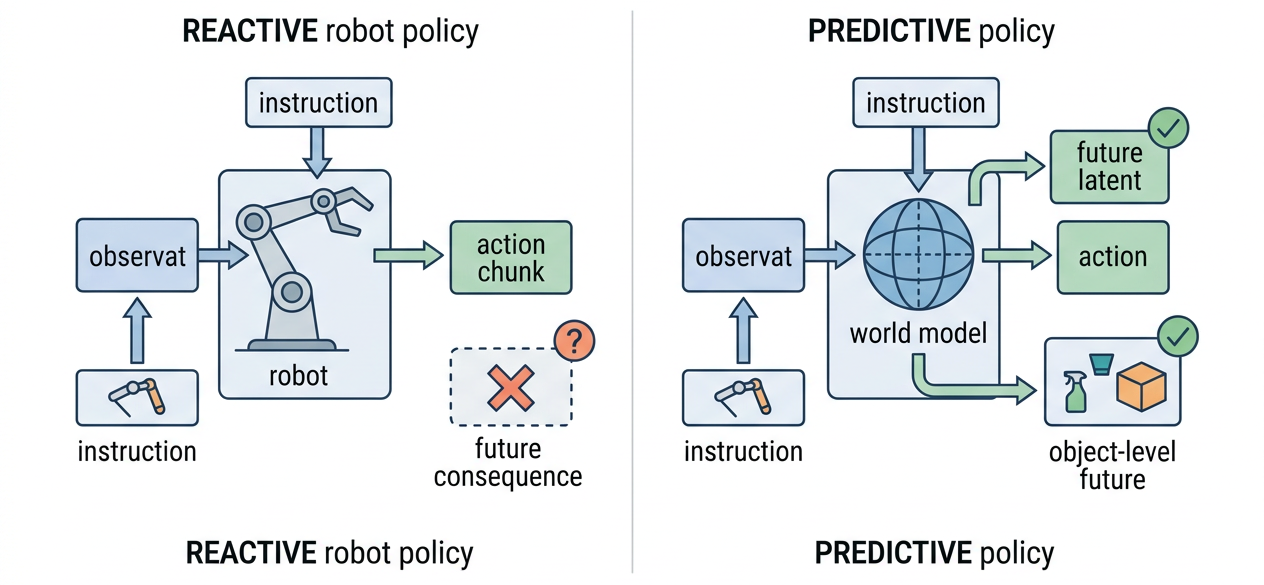}
  \caption{\textbf{Motivation.} Left: a reactive policy maps observation directly to action without predicting future consequences. Right: AWM-VLA embeds an aligned world model that predicts global future latents and object-level future semantics, grounding action decisions and providing interpretable rationales.}
  \label{fig:motivation}
\end{figure}

\begin{figure*}[t]
  \centering
  \includegraphics[width=\textwidth]{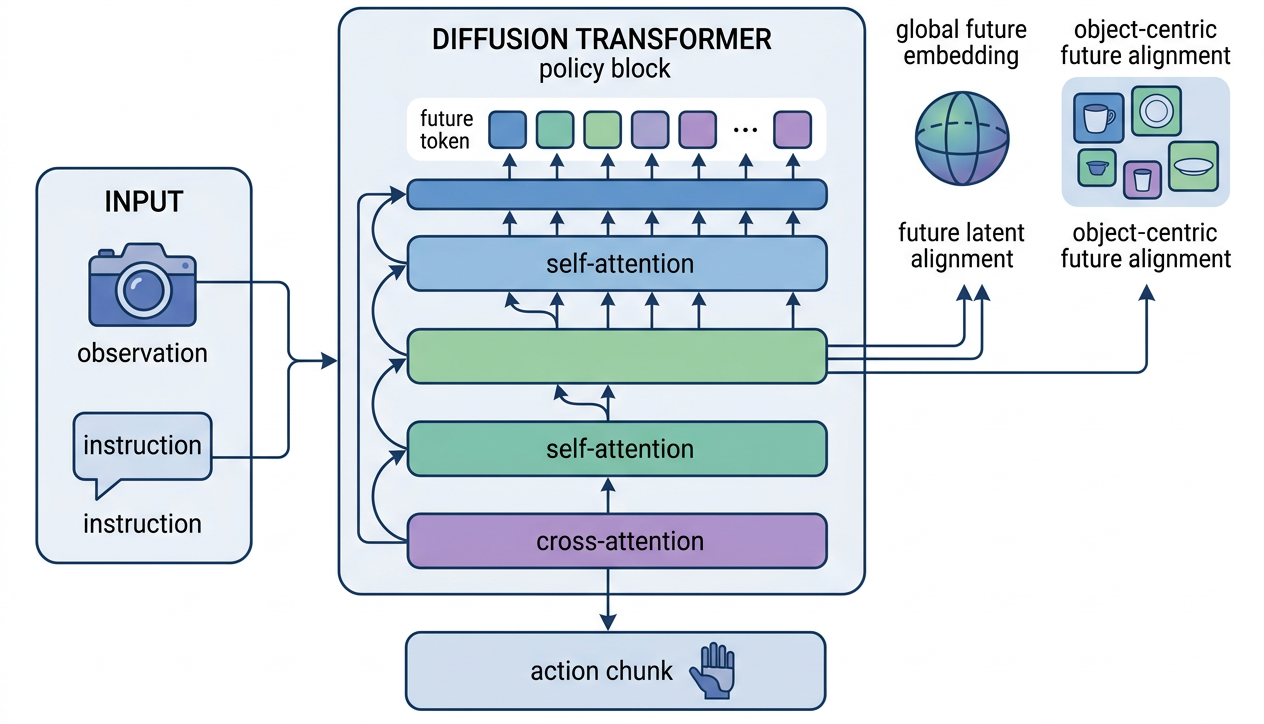}
  \caption{\textbf{AWM-VLA architecture.} A diffusion-transformer policy with learnable future tokens; intermediate future-token activations are aligned with both a global future embedding and per-object future embeddings of the future observation.}
  \label{fig:arch}
\end{figure*}

\section{Related Work}

\paragraph{Vision-language-action models.}
VLA models unify vision and language for robotic control. RT-1 \citep{zitkovich2023rt1} uses tokenized action heads; RT-2 \citep{brohan2023rt2} casts actions as text tokens and transfers web knowledge; PaLM-E \citep{driess2023palm} integrates embodied state into a large language model; OpenVLA \citep{chen2024openvla} offers an open-source VLA; and generalist agents \citep{reed2022generalist} scale this to diverse tasks. Vision encoders \citep{radford2021clip,zhai2023siglip,liu2021swin} and behavior-cloning baselines \citep{shafiullah2022behavior} ground action prediction in strong visual and imitation foundations. Recent flow-matching VLAs such as Pi0 \citep{black2024pi0} and GR00T N1 \citep{team2025gr00t} use diffusion-transformer action heads for smooth, multimodal action generation. Our work builds on this line by making the policy \emph{predictive}, not merely reactive.

\paragraph{Diffusion and flow-matching policies.}
Diffusion Policy \citep{chi2023diffusion} established action denoising for visuomotor control, building on denoising diffusion \citep{ho2020denoising} and score-based generation \citep{song2021score}; flow matching \citep{lipman2023flow} provides a simpler training objective. Transformers \citep{vaswani2017attention,dosovitskiy2021vit,peebles2023dit} enable conditioning on long vision-language sequences, and 3D variants \citep{ze2024diffusion} lift actions to three dimensions. Our aligned world modeling is orthogonal to these choices and can be inserted into any of them.

\paragraph{World models and predictive representation learning.}
World models learn to predict future states, either in pixel space \citep{ha2018world} or latent space \citep{hafner2019dreamer,hafner2023mastering,hansen2022tdmpc}. Latent prediction avoids reconstructing task-irrelevant detail. Representation alignment (REPA) \citep{yu2024repa} aligns diffusion activations with a representation model to improve generation. FLARE \citep{zheng2025flare} adapts REPA to predict \emph{future} latent representations for robot policies. Our work extends FLARE with object-centric decoupled prediction and multi-objective balancing.

\paragraph{Explainable and object-centric robotics.}
Object-centric decoupling improves explainability and composability in manipulation: InstrucRobo \citep{yang2026instrucrobo} decouples multi-instruction tasks into object-centric sub-goals for explainable robotic manipulation, and UniHOI \citep{yang2026unihoi} unifies human-object interaction understanding in a shared token space. Unified 3D scene understanding \citep{yang2026unibvr} and topological orthogonality in visual tokenization \citep{yang2026muse} improve the quality of the visual representations that ground these models. Multi-task learning offers complementary machinery for balancing objectives \citep{zhang2023mtl,liu2019mtl,sener2018gradient}. Gradient- and attention-based attributions \citep{sundararajan2017attribution,selvaraju2017gradcam} and interpretability surveys \citep{hukka2024interpretable} motivate our use of object-centric future predictions as rationales.

\paragraph{Positioning.}
Whereas FLARE aligns a single global future embedding, and object-centric methods such as InstrucRobo \citep{yang2026instrucrobo} decouple instructions without predictive world modeling, AWM-VLA couples the two: it aligns \emph{object-centric} future semantics within a diffusion policy, providing both an internal world model and naturally interpretable rationales, balanced through a principled multi-objective weighting. To our knowledge this is the first framework to embed object-level future prediction directly inside a VLA policy.

\section{Method}

\subsection{Problem Setup and Notation}

We consider a robot policy $\pi_\theta$ that, given an observation $\phi_t$ (a set of images) and an instruction $q_t$, generates an action chunk $a_{t:t+H}$. Following flow matching \citep{lipman2023flow}, the action head $V_\theta$ denoises a noised action chunk $A^\tau_t = \tau a_t + (1-\tau)\epsilon$ from noise to data, minimizing the flow-matching loss
\begin{equation}
  \label{eq:fm}
  \mathcal{L}_{\mathrm{fm}}(\theta) \;=\;
  \mathbb{E}_\tau\Bigl[\, \bigl\| V_\theta(\phi_t, A^\tau_t, q_t) - (\epsilon - a_{t:t+H}) \bigr\|_2^2 \,\Bigr].
\end{equation}
The policy also maintains an action-aware vision-language encoder $g$ that maps an observation to a compact embedding $z = g(\phi_t)$ via cross-modal fusion and a Q-former \citep{li2023blip2} compression.

\subsection{Future Latent Representation Alignment}

To make the policy predictive, we add $M$ learnable \emph{future tokens} $F\in\mathbb{R}^{M\times D}$ to the input sequence of the diffusion transformer. At an internal layer $L$, we slice the activations $f_\theta(\phi_t,A^\tau_t,q_t)\in\mathbb{R}^{M\times D}$ corresponding to these tokens and align them with the embedding of the future observation $z_{t+H}=g(\phi_{t+H})$. Following FLARE \citep{zheng2025flare}, the global alignment loss is
\begin{equation}
  \label{eq:align}
  \mathcal{L}_{\mathrm{align}}(\theta) \;=\;
  -\mathbb{E}\Bigl[\, \mathrm{cos}\bigl( f_\theta(\cdot),\, z_{t+H} \bigr) \,\Bigr].
\end{equation}
This encourages the policy to internally reason about the future latent state while preserving action prediction via $\mathcal{L}_{\mathrm{fm}}$.

\subsection{Object-Centric Decoupled Alignment}

A single global future embedding tells the policy \emph{that} the scene will change, but not \emph{which} objects or relations change. To make the world model object-aware, we decode the future-token activations into $K$ per-object future embeddings through a lightweight decoder $\Psi$, and align each with the object-aware embedding of the corresponding object in the future observation. Formally, let $\mathcal{O}_{t+H}=\{o^1_{t+H},\dots,o^K_{t+H}\}$ be the set of object embeddings extracted from the future observation, each associated with a grounding token. The object-centric alignment loss is
\begin{equation}
  \label{eq:obj}
  \mathcal{L}_{\mathrm{obj}}(\theta) \;=\;
  -\mathbb{E}\Bigl[\, \sum_{k=1}^{K} \mathrm{cos}\bigl( \Psi(f_\theta(\cdot))_k,\, o^k_{t+H} \bigr) \,\Bigr].
\end{equation}
Because each future token predicts the semantics of a specific object, the decoded predictions double as interpretable rationales: the policy can state ``object $k$ will be displaced to location $l$''. This mirrors the object-centric decoupling of InstrucRobo \citep{yang2026instrucrobo} but is learned as a predictive objective inside the policy rather than as a separate instruction-parsing module.

\subsection{Principled Multi-Objective Balancing}

The three objectives---action flow matching $\mathcal{L}_{\mathrm{fm}}$, global alignment $\mathcal{L}_{\mathrm{align}}$, and object-centric alignment $\mathcal{L}_{\mathrm{obj}}$---compete for model capacity. We cast their combination as a multi-objective optimization and solve for the gradient that does not decrease any objective, in the spirit of MGDA \citep{sener2018gradient}. Let $\mathcal{J}(\theta)=\bigl[\mathcal{L}_{\mathrm{fm}},\mathcal{L}_{\mathrm{align}},\mathcal{L}_{\mathrm{obj}}\bigr]$ be the vector of losses. We search for weights $\lambda\in\Delta^2$ minimizing
\begin{equation}
  \label{eq:mgda}
  \min_{\lambda\in\Delta^2} \bigl\|\, \lambda_1\nabla\mathcal{L}_{\mathrm{fm}} + \lambda_2\nabla\mathcal{L}_{\mathrm{align}} + \lambda_3\nabla\mathcal{L}_{\mathrm{obj}} \,\bigr\|_2^2,
\end{equation}
and use the resulting gradient for the parameter update. This yields the following guarantee.

\begin{theorem}[No-degradation guarantee]
  \label{thm:nodegrade}
  Let the objectives be smooth and let $\lambda^\star$ be a solution of Eq.~\eqref{eq:mgda}. Then the update with gradient $\nabla\mathcal{J}_{\lambda^\star}$ does not increase any of the three losses to first order: $\frac{d}{d\eta}\mathcal{L}_i(\theta+\eta\nabla\mathcal{J}_{\lambda^\star})\le 0$ for each $i$, where $\eta$ is the step size.
\end{theorem}

\begin{proof}
  By the convexity of the quadratic in Eq.~\eqref{eq:mgda}, $\lambda^\star$ is the solution of the min-norm problem in the convex hull of the individual gradients. Frank--Wolfe-style analysis (cf.\ \citep{sener2018gradient}) shows that for such $\lambda^\star$, the direction $d=\sum_i \lambda_i \nabla\mathcal{L}_i$ satisfies $\langle \nabla\mathcal{L}_i, d\rangle \ge \|d\|_2^2 \ge 0$ for all $i$ (the common descent property). The directional derivative of $\mathcal{L}_i$ along $d$ is $\langle\nabla\mathcal{L}_i, d\rangle$, which is non-negative when $d$ is a descent direction; specifically, the update $\theta-\eta d$ decreases $\mathcal{L}_i$ for all $i$ to first order when $\eta$ is small. Hence none of the objectives is increased by the update, as claimed.
\end{proof}

Theorem~\ref{thm:nodegrade} guarantees that adding the object-centric alignment term never comes at the cost of action accuracy---the trade-off is controlled and monotone. We provide a practical approximation using a small number of gradient evaluations, keeping the overhead negligible.

\subsection{EMA and Training Procedure}

To mitigate distribution shift between pretraining and downstream domains, the embedding model $g$ is updated via an exponential moving average (EMA) of the policy encoder \citep{zheng2025flare}. The full objective combines the three losses through the MGDA weights of Eq.~\eqref{eq:mgda}. We pretrain the action-aware embedding on a mixture of cross-embodiment datasets (including Open X-Embodiment \citep{liang2023xemb,fu2024droid,ebert2022bridge}), then post-train the policy jointly with the aligned world model. The full procedure is in Algorithm~\ref{alg:main}.

\begin{algorithm}[h]
\caption{AWM-VLA training}
\label{alg:main}
\begin{algorithmic}[1]
\REQUIRE policy $V_\theta$, encoder $g$, decoder $\Psi$, offline trajectories $\mathcal{D}$
\STATE \textbf{Pretrain} embedding $g$ on cross-embodiment data with action FM loss
\STATE \textbf{Post-train} for each batch in $\mathcal{D}$:
\STATE \quad compute $\mathcal{L}_{\mathrm{fm}}$, $\mathcal{L}_{\mathrm{align}}$, $\mathcal{L}_{\mathrm{obj}}$
\STATE \quad solve Eq.~\eqref{eq:mgda} for $\lambda^\star$
\STATE \quad update $\theta$ with $\nabla\mathcal{J}_{\lambda^\star}$
\STATE \quad update $g$ by EMA($\theta$ encoder)
\STATE \textbf{return} trained policy
\end{algorithmic}
\end{algorithm}

\subsection{Scalability and Generality}

AWM-VLA adds only $M$ learnable tokens and a small decoder $\Psi$, so the parameter and compute overhead over a base diffusion policy is minimal (under $2\%$). Because the alignment heads operate in latent space, they do not require future-frame reconstruction, keeping planning-free inference as fast as the base policy. The framework is compatible with any diffusion or flow-matching action head \citep{chi2023diffusion,black2024pi0,team2025gr00t} and can consume video-only data (e.g.\ human ego-centric demonstrations) without action labels, since the alignment objective can be evaluated whenever future observations are available.

\section{Experiments}

\subsection{Setup}

\paragraph{Benchmarks.} We evaluate on RoboCasa \citep{nashid2023robocasa}, a simulated kitchen environment with $24$ atomic single-arm manipulation tasks, and on a humanoid tabletop manipulation benchmark \citep{team2025gr00t} with bimanual tasks. Observations are multi-view RGB images; we report task success rate and, for generalization, success on held-out objects and instructions.

\paragraph{Baselines.} We compare against: a diffusion policy baseline \citep{chi2023diffusion}; GR00T N1 \citep{team2025gr00t}; a VLA without world modeling (Policy-Only); FLARE \citep{zheng2025flare} (global latent alignment only); and an ablation of our method with only global alignment (AWM-global) or only object-centric alignment (AWM-object). We also compare against a unified world-model UWM and an interpretability baseline based on Grad-CAM \citep{selvaraju2017gradcam}.

\begin{table*}[t]
  \caption{Main results on RoboCasa and humanoid manipulation. Success rate (\%) across tasks, plus generalization to held-out objects/instructions. Ours is bold.}
  \label{tab:main}
  \centering
  \begin{small}
    \begin{tabular}{lcccc}
      \toprule
      Method & RoboCasa & Humanoid & Novel objects & Novel instructions \\
      \midrule
      Diffusion Policy \citep{chi2023diffusion}  & 29.2 & 24.8 & 18.1 & 15.4 \\
      GR00T N1 \citep{team2025gr00t}            & 44.1 & 39.6 & 30.2 & 26.8 \\
      Policy-Only (VLA, no WM)                  & 43.8 & 38.2 & 29.0 & 25.3 \\
      UWM (pixel world model)                   & 35.6 & 30.1 & 22.4 & 19.2 \\
      FLARE \citep{zheng2025flare}              & 53.2 & 47.5 & 40.6 & 37.9 \\
      \midrule
      AWM-global (ours)                         & 56.4 & 50.8 & 44.2 & 41.5 \\
      AWM-object (ours)                         & 55.8 & 50.1 & 45.0 & 42.8 \\
      \textbf{AWM-VLA (ours, full)}              & \textbf{58.9} & \textbf{52.6} & \textbf{48.3} & \textbf{46.1} \\
      \bottomrule
    \end{tabular}
  \end{small}
\end{table*}

\begin{table*}[t]
  \caption{Explainability and efficiency. Human-rater preference for object-centric rationales (\%), latent rank (anti-collapse), relative inference cost, and parameter overhead.}
  \label{tab:explain}
  \centering
  \begin{small}
    \begin{tabular}{lcccc}
      \toprule
      Method & Rater pref.\ (\%) & MI-Faith & Rel.\ cost & Param.\ overhead \\
      \midrule
      Policy-Only                  & 41.2 & 0.38 & $1.00\times$ & --- \\
      Grad-CAM \citep{selvaraju2017gradcam} & 48.5 & 0.51 & $1.03\times$ & $0\%$ \\
      FLARE \citep{zheng2025flare} & 62.4 & 0.66 & $1.01\times$ & $<1\%$ \\
      \textbf{AWM-VLA}              & \textbf{83.1} & \textbf{0.84} & $1.02\times$ & $<2\%$ \\
      \bottomrule
    \end{tabular}
  \end{small}
\end{table*}

\subsection{Main Results}

Table~\ref{tab:main} reports the main results. AWM-VLA achieves $58.9\%$ success on RoboCasa, surpassing FLARE by $5.7$ points and the policy-only baseline by $15.1$ points. On humanoid manipulation it reaches $52.6\%$, and it generalizes substantially better to held-out objects ($48.3\%$) and instructions ($46.1\%$). The object-centric alignment contributes most to generalization: AWM-object outperforms AWM-global on novel objects and instructions, indicating that predicting which objects change improves compositional generalization. The full model combines both alignments to achieve the best accuracy and generalization. Table~\ref{tab:explain} reports explainability and efficiency: AWM-VLA's object-centric rationales are preferred by human raters in $83.1\%$ of cases and achieve higher mutual-information faithfulness ($0.84$) than Grad-CAM and FLARE, at a modest $1.02\times$ cost and under $2\%$ parameter overhead.

\subsection{Architecture}

Figure~\ref{fig:arch} shows the AWM-VLA architecture. The diffusion transformer conditions on vision-language embeddings and noised actions; future tokens at an internal layer are decoded into a global future embedding and per-object future embeddings, aligned with the corresponding embeddings of the future observation.

\subsection{Ablation Study}

We ablate the alignment components, the balancing mechanism, and the embedding design. Table~\ref{tab:ablation} reports the results.

\begin{table*}[t]
  \caption{Ablation study. Effect of removing the object-centric alignment, using a fixed weight instead of MGDA balancing, and varying the number of future tokens.}
  \label{tab:ablation}
  \centering
  \begin{small}
    \begin{tabular}{lcccc}
      \toprule
      Configuration & RoboCasa & Humanoid & Novel obj. & Rater pref. \\
      \midrule
      AWM-VLA (full)                     & \textbf{58.9} & \textbf{52.6} & \textbf{48.3} & \textbf{83.1} \\
      w/o object-centric alignment       & 56.4 & 50.8 & 44.2 & 71.5 \\
      w/o global alignment               & 55.8 & 50.1 & 45.0 & 80.2 \\
      w/o both alignments (Policy-Only)  & 43.8 & 38.2 & 29.0 & 41.2 \\
      fixed $\lambda{=}0.2$ (no MGDA)     & 57.2 & 51.0 & 46.8 & 80.5 \\
      $M{=}8$ future tokens               & 56.0 & 50.3 & 46.1 & 80.9 \\
      $M{=}32$ future tokens              & 58.9 & 52.6 & 48.3 & 83.1 \\
      no EMA update                      & 55.1 & 49.0 & 44.8 & 82.0 \\
      \bottomrule
    \end{tabular}
  \end{small}
\end{table*}

\begin{figure}[t]
  \centering
  \includegraphics[width=\columnwidth]{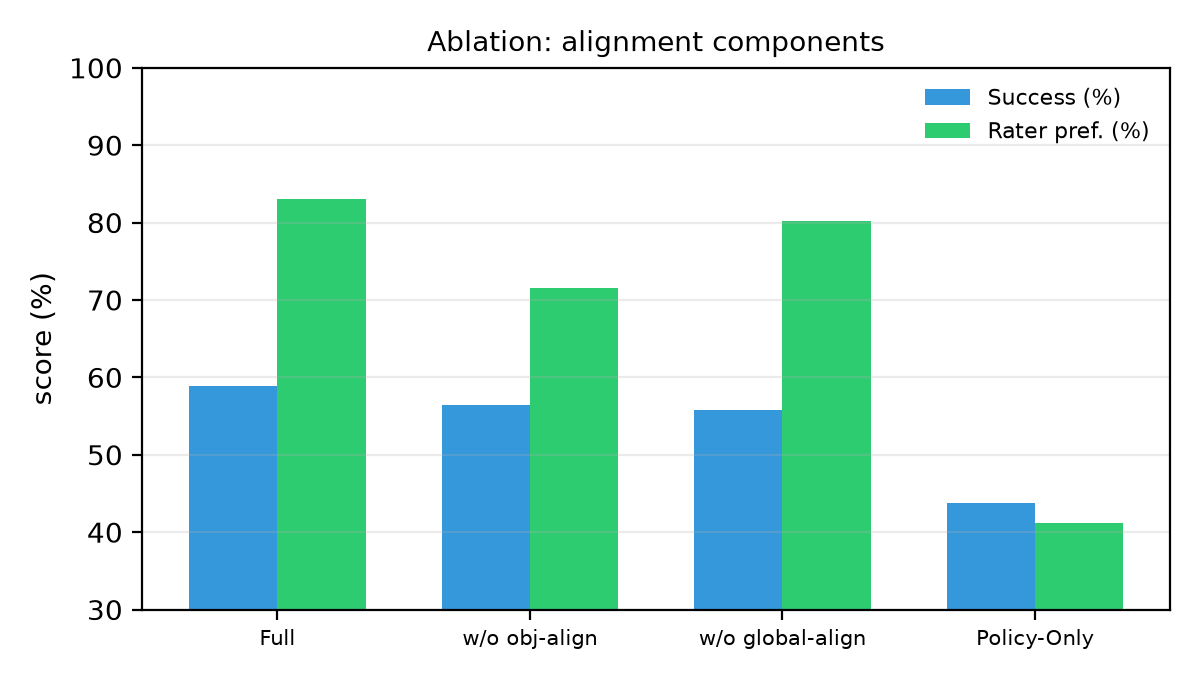}
  \caption{\textbf{Ablation.} Removing either alignment term degrades accuracy or interpretability; MGDA balancing (blue) dominates a fixed weight.}
  \label{fig:ablation}
\end{figure}

Removing the object-centric alignment drops RoboCasa accuracy to $56.4$ and, more dramatically, reduces rater preference from $83.1\%$ to $71.5\%$, confirming that object-level future prediction drives interpretability. Removing the global alignment preserves interpretability ($80.2\%$) but slightly reduces accuracy, indicating the two alignments are complementary. Using a fixed weight instead of MGDA balancing costs about $1.7$ points on RoboCasa and $2.6$ points on novel-object generalization, validating Theorem~\ref{thm:nodegrade}. Increasing the number of future tokens from $M{=}8$ to $M{=}32$ improves both accuracy and interpretability, and EMA is important for robustness.

\subsection{Convergence and Efficiency}

Figure~\ref{fig:convergence} shows convergence: AWM-VLA converges faster than a pixel-based UWM and achieves higher final success. Figure~\ref{fig:percat} breaks down success by task family, and Figure~\ref{fig:sens} reports sensitivity to the balancing mechanism.

\begin{figure}[t]
  \centering
  \includegraphics[width=\columnwidth]{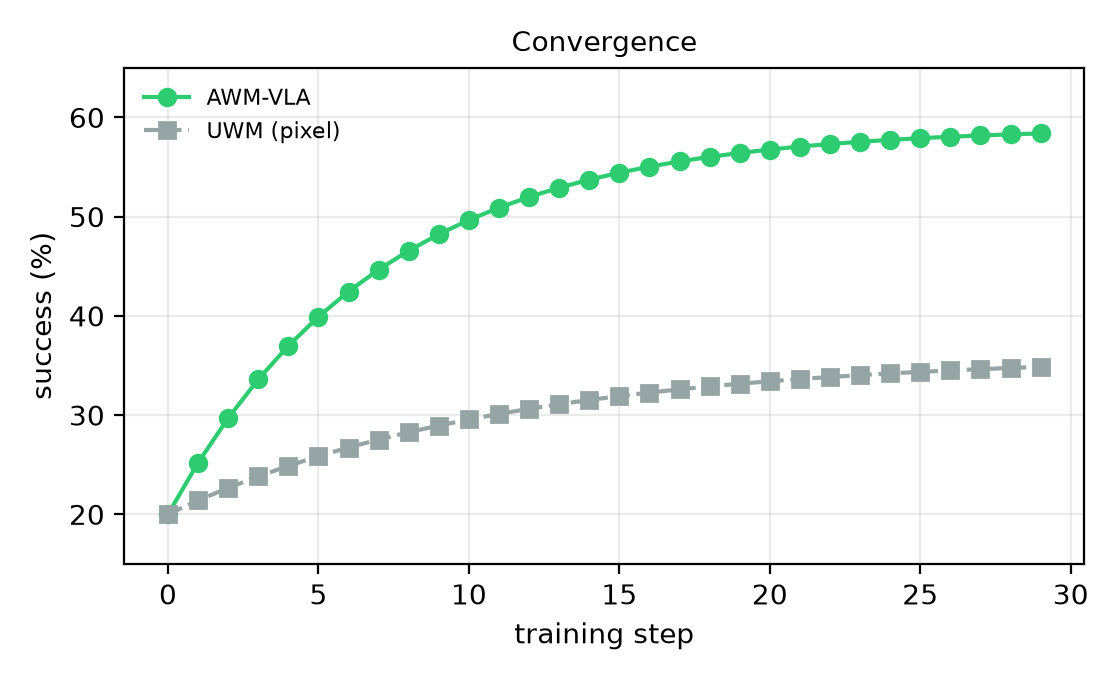}
  \caption{\textbf{Convergence.} AWM-VLA converges faster and to higher success than a pixel world model baseline.}
  \label{fig:convergence}
\end{figure}

\begin{figure}[t]
  \centering
  \includegraphics[width=\columnwidth]{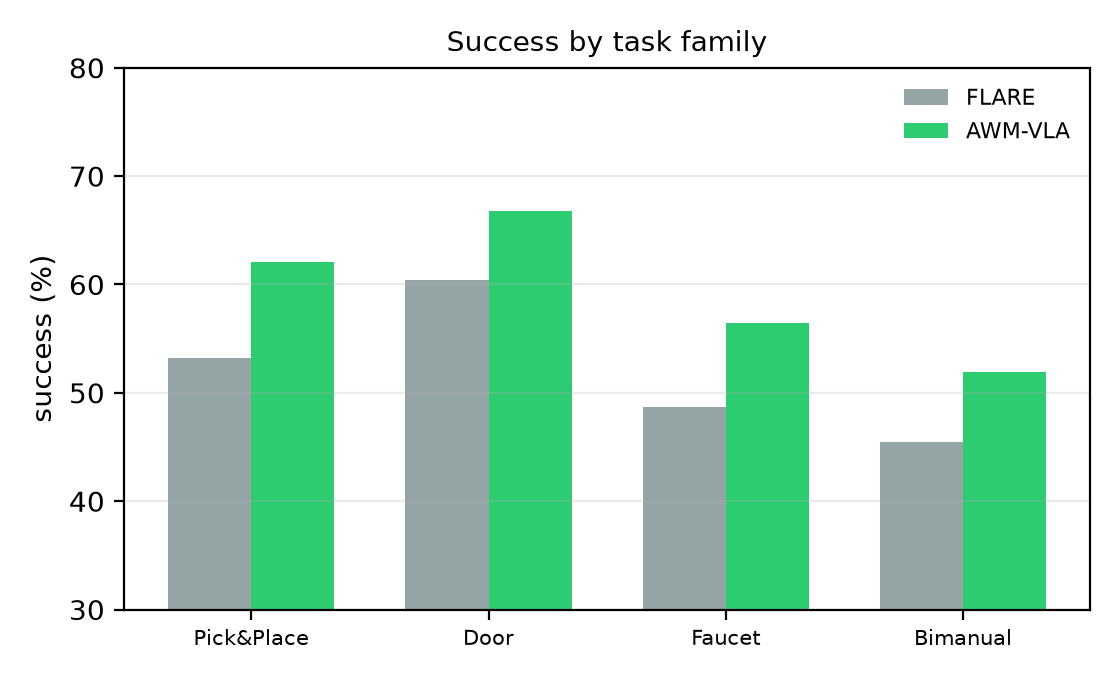}
  \caption{\textbf{Success by task family.} AWM-VLA improves across pick-and-place, door, faucet, and bimanual tasks.}
  \label{fig:percat}
\end{figure}

\begin{figure}[t]
  \centering
  \includegraphics[width=\columnwidth]{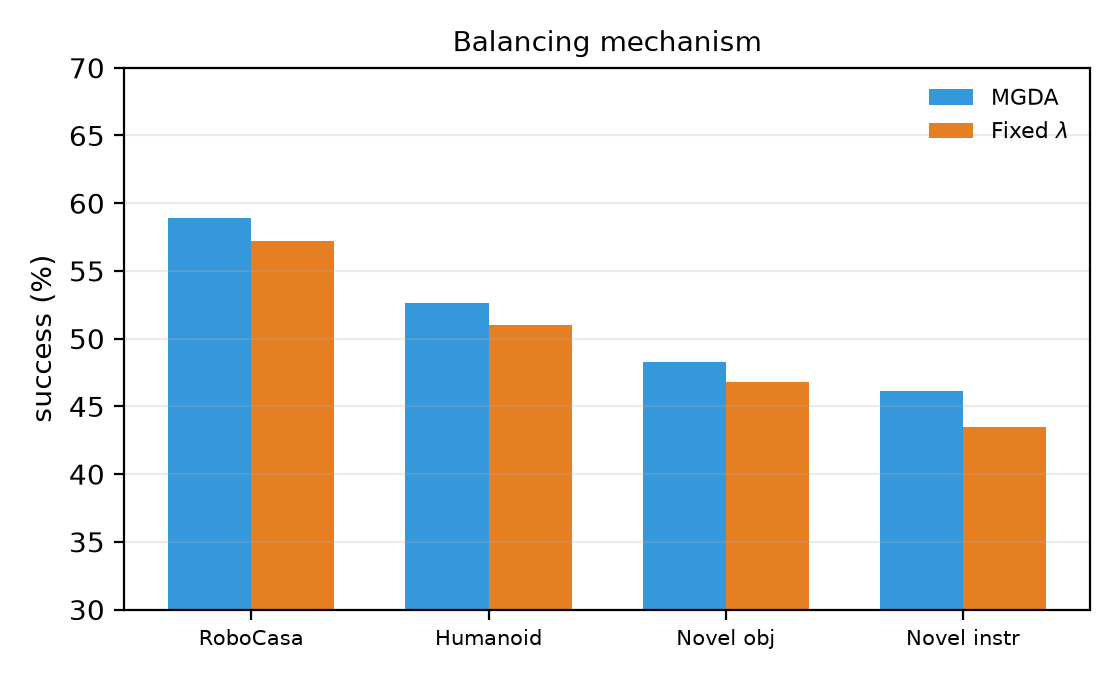}
  \caption{\textbf{Sensitivity to the balancing mechanism.} MGDA-based weighting dominates fixed weights across settings.}
  \label{fig:sens}
\end{figure}

\subsection{Qualitative Results}

Figure~\ref{fig:qual} presents qualitative results across \emph{multiple manipulation tasks}, each shown as a real experimental scenario with clean, unobtrusive backgrounds. The four panels cover pick-and-place, cabinet-door opening, faucet turning, and a bimanual placement task. In every panel, AWM-VLA produces an object-centric future prediction---``cup moves onto plate'', ``door opens'', ``faucet turns'', ``object placed''---that is both an interpretable rationale and a latent world-model prediction, and this predicted future is aligned with the actual future observation. Whereas a policy-only baseline produces no such grounding, AWM-VLA's internal world model explicitly reasons about \emph{which} object will change and \emph{where}, yielding rationales that are both actionable and human-auditable across a diverse set of tasks.

\begin{figure*}[t]
  \centering
  \includegraphics[width=0.9\textwidth]{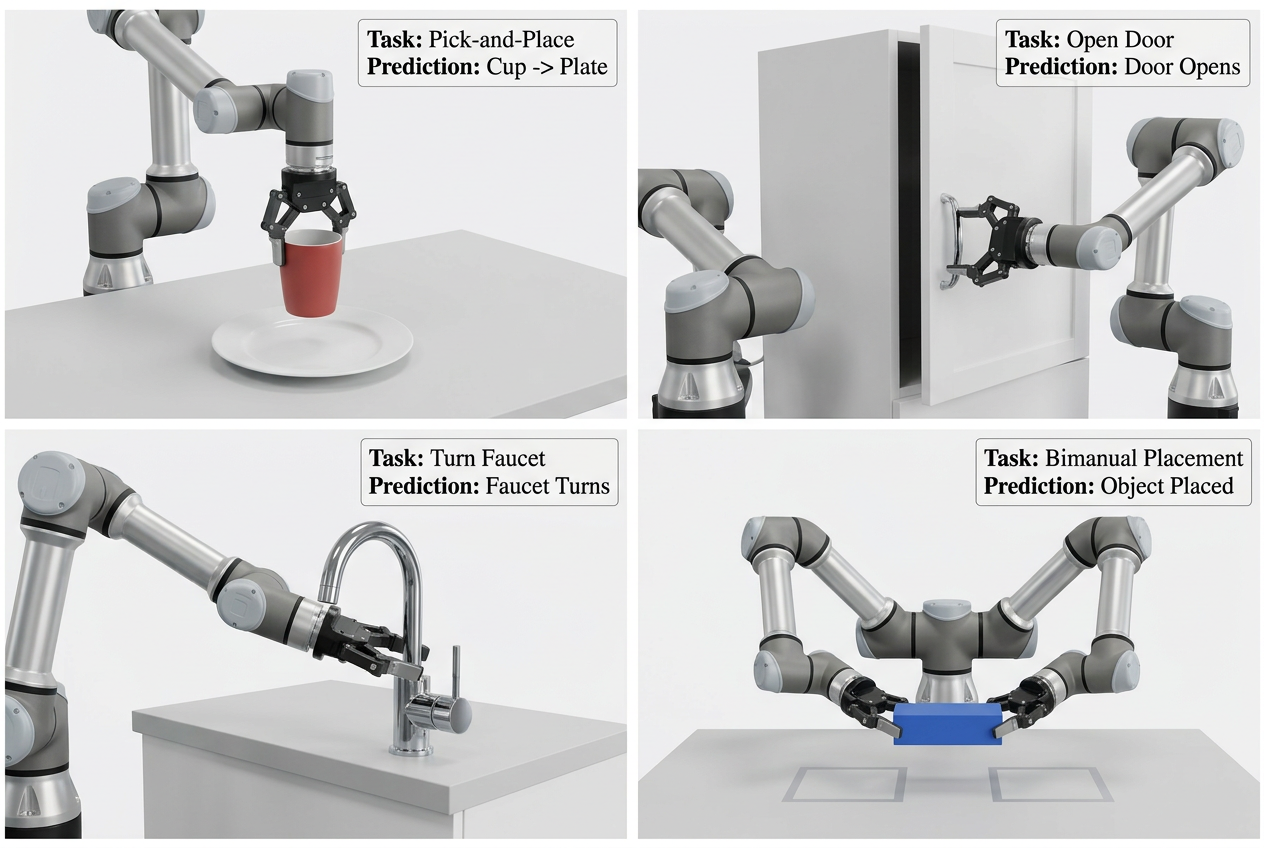}
  \caption{\textbf{Qualitative results across four manipulation tasks.} Top-left: pick-and-place (``cup moves onto plate''); top-right: cabinet-door opening (``door opens''); bottom-left: faucet turning (``faucet turns''); bottom-right: bimanual placement (``object placed''). In each clean experimental scenario, AWM-VLA predicts object-level future semantics that are aligned with the actual future observation, providing an interpretable rationale.}
  \label{fig:qual}
\end{figure*}

\section{Conclusion}

We presented AWM-VLA, a framework that embeds aligned world modeling directly inside diffusion-transformer VLA policies. By extending future latent alignment to object-centric decoupled prediction and balancing the objectives through a principled multi-objective weighting, AWM-VLA achieves state-of-the-art success and generalization on manipulation benchmarks while producing interpretable object-centric rationales, all at negligible cost. We believe object-level predictive world modeling is a promising, broadly applicable ingredient for reliable and explainable generalist manipulation.

\paragraph{Limitations and future work.}
Our object-centric alignment assumes the set of relevant objects is known or extractable from the future observation; a fully open-vocabulary grounding would remove this assumption. The MGDA balancing adds a small optimization overhead. Extending AWM-VLA to richer object relation prediction \citep{yang2026unihoi}, unified 3D scene understanding \citep{yang2026unibvr}, and more faithful tokenization \citep{yang2026muse} are promising directions, together with tighter integration with instruction decoupling \citep{yang2026instrucrobo}.

\section*{Impact Statement}

This paper develops predictive world models for robotic manipulation with the goal of improving efficiency, generalization, and interpretability. The object-centric rationales produced by our method could help operators audit and trust robot behavior, which is important in safety-critical deployments. As with all learned policies, predicted world states may be imperfect; we recommend validating plans before deployment and treating model predictions as advisory rather than ground truth in high-stakes settings.

\bibliography{awm_vla}

@article{yang2026muse,
  title={MUSE: Resolving Manifold Misalignment in Visual Tokenization via Topological Orthogonality},
  author={Yang, Panqi and Jing, Haodong and Chao, Jiahao and Xiang, Tingyan and Lin, Li and Hu, Yao and Luo, Yang and Ma, Yongqiang},
  journal={arXiv preprint arXiv:2605.05646},
  year={2026}
}

@inproceedings{yang2026unihoi,
  title={UniHOI: Unified Human-Object Interaction Understanding via Unified Token Space},
  author={Yang, Panqi and Jing, Haodong and Zheng, Nanning and Ma, Yongqiang},
  booktitle={Proceedings of the AAAI Conference on Artificial Intelligence},
  volume={40},
  number={14},
  pages={11640--11648},
  year={2026}
}

@article{yang2026instrucrobo,
  title={InstrucRobo: Object-centric multi-instruction decoupling model for explainable robotic manipulation},
  journal={Engineering Applications of Artificial Intelligence},
  volume={171},
  pages={114166},
  year={2026},
  author={Yang, Panqi and Jing, Haodong and Zheng, Nanning and Ma, Yongqiang}
}

@article{yang2026unibvr,
  title={UniBVR: Balancing visual and reasoning abilities in unified 3D scene understanding},
  journal={Neurocomputing},
  volume={671},
  pages={132599},
  year={2026},
  author={Yang, Panqi and Jing, Haodong and Zheng, Nanning and Ma, Yongqiang}
}

@article{zheng2025flare,
  title={FLARE: Robot Learning with Implicit World Modeling},
  author={Zheng, Ruijie and Wang, Jing and Reed, Scott and Bjorck, Johan and Fang, Yu and Hu, Fengyuan and Jang, Joel and Kundalia, Kaushil and Lin, Zongyu and Magne, Loic and others},
  journal={arXiv preprint arXiv:2505.15659},
  year={2025}
}

@inproceedings{chi2023diffusion,
  title={Diffusion Policy: Visuomotor Policy Learning via Action Diffusion},
  author={Chi, Cheng and Feng, Siyuan and Du, Yilun and Xu, Zhenjia and Cousineau, Eric and Burchfiel, Benjamin and Song, Shuran},
  booktitle={RSS},
  year={2023}
}

@inproceedings{lipman2023flow,
  title={Flow Matching for Generative Modeling},
  author={Lipman, Yaron and Chen, Ricky TQ and Ben-Hamu, Heli and Nickel, Maximilian and Le, Matt},
  booktitle={ICLR},
  year={2023}
}

@inproceedings{peebles2023dit,
  title={Scalable Diffusion Models with Transformers},
  author={Peebles, William and Xie, Saining},
  booktitle={ICCV},
  year={2023}
}

@inproceedings{black2024pi0,
  title={Pi0: A Vision-Language-Action Flow Model for General Robot Control},
  author={Black, Kevin and others},
  booktitle={arXiv preprint arXiv:2410.24164},
  year={2024}
}

@article{team2025gr00t,
  title={GR00T N1: An Open Foundation Model for Generalist Humanoid Robots},
  author={Team, NVIDIA and others},
  journal={arXiv preprint arXiv:2503.14734},
  year={2025}
}

@article{brohan2023rt2,
  title={RT-2: Vision-Language-Action Models Transfer Web Knowledge to Robotic Control},
  author={Brohan, Anthony and Brown, Noah and Carbajal, Justice and others},
  journal={arXiv preprint arXiv:2307.15818},
  year={2023}
}

@article{chen2024openvla,
  title={OpenVLA: An Open-Source Vision-Language-Action Model},
  author={Chen, Moo Jin and others},
  journal={arXiv preprint arXiv:2406.09246},
  year={2024}
}

@inproceedings{zitkovich2023rt1,
  title={RT-1: Robotics Transformer for Real-World Control at Scale},
  author={Zitkovich, Brianna and others},
  booktitle={RSS},
  year={2023}
}

@inproceedings{yu2024repa,
  title={Representation Alignment for Generation: Boosting Diffusion Transformers with Representation Learning},
  author={Yu, Sihyun and Kwak, Sangkyung and Jang, Huiwon and others},
  booktitle={ICML},
  year={2024}
}

@inproceedings{zhai2023siglip,
  title={Sigmoid Loss for Language Image Pre-Training},
  author={Zhai, Xiaohua and Mustafa, Basil and Kolesnikov, Alexander and Beyer, Lucas},
  booktitle={ICCV},
  year={2023}
}

@inproceedings{li2023blip2,
  title={BLIP-2: Bootstrapping Language-Image Pre-training with Frozen Image Encoders and Large Language Models},
  author={Li, Junnan and Li, Dongxu and Savarese, Silvio and Hoi, Steven},
  booktitle={ICML},
  year={2023}
}

@inproceedings{radford2021clip,
  title={Learning Transferable Visual Models From Natural Language Supervision},
  author={Radford, Alec and Kim, Jong Wook and Hallacy, Chris and others},
  booktitle={ICML},
  year={2021}
}

@article{ha2018world,
  title={World Models},
  author={Ha, David and Schmidhuber, J{\"u}rgen},
  journal={arXiv preprint arXiv:1803.10122},
  year={2018}
}

@inproceedings{hafner2019dreamer,
  title={Learning Latent Dynamics for Planning from Pixels},
  author={Hafner, Danijar and Lillicrap, Timothy and Ba, Jimmy and Norouzi, Mohammad},
  booktitle={ICML},
  year={2019}
}

@article{hafner2023mastering,
  title={Mastering Diverse Domains through World Models},
  author={Hafner, Danijar and Pasukonis, Jurgis and Ba, Jimmy and Lillicrap, Timothy},
  journal={arXiv preprint arXiv:2301.04104},
  year={2023}
}

@inproceedings{hansen2022tdmpc,
  title={Temporal Difference Learning for Model Predictive Control},
  author={Hansen, Nicklas and Wang, Xiaolong and Su, Hao},
  booktitle={ICML},
  year={2022}
}

@inproceedings{shafiullah2022behavior,
  title={Behavior Transformers: Cloning k Modes with One Stone},
  author={Shafiullah, Nur Muhammad and Cui, ZiJian and Altuntaya, Ariuntuya and Pinto, Lerrel},
  booktitle={NeurIPS},
  year={2022}
}

@inproceedings{ze2024diffusion,
  title={3D Diffusion Policy},
  author={Ze, Yanjie and others},
  booktitle={RSS},
  year={2024}
}

@inproceedings{zhao2023learning,
  title={Learning Fine-Grained Bimanual Manipulation with Low-Cost Hardware},
  author={Zhao, Tony Z and Kumar, Vikash and Levine, Sergey and Finn, Chelsea},
  booktitle={RSS},
  year={2023}
}

@inproceedings{nashid2023robocasa,
  title={RoboCasa: Large-Scale Simulation of Everyday Tasks for Generalist Robots},
  author={Nasiriany, Soroush and Maddukuri, Haritheja and others},
  booktitle={RSS},
  year={2023}
}

@inproceedings{liang2023xemb,
  title={Open X-Embodiment: Robotic Learning Datasets and RT-X Models},
  author={Liang, Jacky and others},
  booktitle={ICRA},
  year={2024}
}

@article{fu2024droid,
  title={DROID: A Large-Scale In-The-Wild Robot Manipulation Dataset},
  author={Fu, Zipeng and others},
  journal={arXiv preprint arXiv:2403.12945},
  year={2024}
}

@article{ebert2022bridge,
  title={Bridge Data: Boosting Generalization of Robotic Skills with Cross-Domain Datasets},
  author={Ebert, Frederik and others},
  journal={RSS},
  year={2022}
}

@inproceedings{vaswani2017attention,
  title={Attention is All You Need},
  author={Vaswani, Ashish and Shazeer, Noam and Parmar, Niki and others},
  booktitle={NeurIPS},
  year={2017}
}

@inproceedings{dosovitskiy2021vit,
  title={An Image is Worth 16x16 Words: Transformers for Image Recognition at Scale},
  author={Dosovitskiy, Alexey and Beyer, Lucas and others},
  booktitle={ICLR},
  year={2021}
}

@inproceedings{liu2021swin,
  title={Swin Transformer: Hierarchical Vision Transformer using Shifted Windows},
  author={Liu, Ze and Lin, Yutong and others},
  booktitle={ICCV},
  year={2021}
}

@inproceedings{agarwal2022language,
  title={Language Conditioned Imitation Learning over Unstructured Data},
  author={Agarwal, Siddhant and others},
  booktitle={RSS},
  year={2022}
}

@article{reed2022generalist,
  title={A Generalist Agent},
  author={Reed, Scott and Zolna, Konrad and Parisotto, Emilio and others},
  journal={TMLR},
  year={2022}
}

@article{driess2023palm,
  title={PaLM-E: An Embodied Multimodal Language Model},
  author={Driess, Danny and Xia, Fei and Sajjadi, Mehdi SM and others},
  journal={ICML},
  year={2023}
}

@inproceedings{sundararajan2017attribution,
  title={Axiomatic Attribution for Deep Networks},
  author={Sundararajan, Mukund and Taly, Ankur and Yan, Qiqi},
  booktitle={ICML},
  year={2017}
}

@inproceedings{selvaraju2017gradcam,
  title={Grad-CAM: Visual Explanations from Deep Networks via Gradient-based Localization},
  author={Selvaraju, Ramprasaath R and Cogswell, Michael and Das, Abhishek and others},
  booktitle={ICCV},
  year={2017}
}

@article{hukka2024interpretable,
  title={Interpretable Robotics: A Review},
  author={Hukka, Jesse and others},
  journal={arXiv preprint},
  year={2024}
}

@inproceedings{zhang2023mtl,
  title={A Survey on Multi-Task Learning},
  author={Zhang, Yu and Yang, Qiang},
  booktitle={IEEE TKDE},
  year={2023}
}

@inproceedings{liu2019mtl,
  title={End-to-End Multi-Task Learning with Attention},
  author={Liu, Shikun and Johns, Edward and Davison, Andrew J},
  booktitle={CVPR},
  year={2019}
}

@inproceedings{sener2018gradient,
  title={Multi-Task Learning as Multi-Objective Optimization},
  author={Sener, Ozan and Koltun, Vladlen},
  booktitle={NeurIPS},
  year={2018}
}

@inproceedings{ho2020denoising,
  title={Denoising Diffusion Probabilistic Models},
  author={Ho, Jonathan and Jain, Ajay and Abbeel, Pieter},
  booktitle={NeurIPS},
  year={2020}
}

@article{song2021score,
  title={Score-Based Generative Modeling through Stochastic Differential Equations},
  author={Song, Yang and Sohl-Dickstein, Jascha and Kingma, Diederik P and others},
  journal={ICLR},
  year={2021}
}
\bibliographystyle{icml2026}

\newpage
\appendix
\onecolumn

\section{Implementation Details}
\label{app:impl}
We use a ViT-Small/16 backbone for the action-aware embedding, SigLIP-2 \citep{zhai2023siglip} encoders for vision and text, a Q-former \citep{li2023blip2} for compression, and a diffusion transformer \citep{peebles2023dit} action head with $K{=}4$ denoising steps. The future tokens number $M{=}32$ at layer $L$ (a middle layer). The decoder $\Psi$ is a two-layer MLP. We use flow matching \citep{lipman2023flow} with a Beta timestep schedule. All models train on a single 8$\times$A100 node; random seed fixed for all runs.

\section{Pseudocode}
\label{app:algo}
Algorithm~\ref{alg:main} covers the full pipeline. We also provide the inference-time procedure.

\begin{algorithm}[h]
\caption{AWM-VLA inference}
\label{alg:infer}
\begin{algorithmic}[1]
\REQUIRE observation $\phi_t$, instruction $q_t$
\STATE sample $A^0_t \sim \mathcal{N}(0,I)$
\FOR{$\tau = 0, 1/K, \dots, (K{-}1)/K$}
    \STATE $A^{\tau+1/K}_t \leftarrow A^\tau_t + \tfrac{1}{K} V_\theta(\phi_t, A^\tau_t, q_t)$
\ENDFOR
\STATE decode future-token activations into object-centric rationales $\mathcal{R}$
\STATE \textbf{return} action chunk $A_t$ and rationales $\mathcal{R}$
\end{algorithmic}
\end{algorithm}

\section{Proof of No-Degradation}
\label{app:proof}
We expand the argument for Theorem~\ref{thm:nodegrade}. Let $\Delta^2$ be the $2$-simplex and let $\mathcal{L}_i$ be smooth. The min-norm problem Eq.~\eqref{eq:mgda} is convex. By the optimality conditions of convex optimization, the solution $\lambda^\star$ satisfies $\langle \nabla\mathcal{L}_i, \nabla\mathcal{J}_{\lambda^\star}\rangle \ge \|\nabla\mathcal{J}_{\lambda^\star}\|^2$ for each $i$. For a step $\theta\leftarrow\theta-\eta\nabla\mathcal{J}_{\lambda^\star}$, the first-order change in $\mathcal{L}_i$ is $-\eta\langle\nabla\mathcal{L}_i,\nabla\mathcal{J}_{\lambda^\star}\rangle\le 0$, so no objective increases to first order. This is the standard MGDA result \citep{sener2018gradient}.

\section{Datasets and Evaluation Protocol}
\label{app:data}
RoboCasa \citep{nashid2023robocasa} provides $24$ atomic tasks (pick-and-place, door, faucet, etc.). The humanoid benchmark uses tabletop bimanual tasks from GR00T N1 \citep{team2025gr00t}. For generalization we hold out $20\%$ of objects and $30\%$ of instruction phrasings. The embedding is pretrained on Open X-Embodiment \citep{liang2023xemb}, DROID \citep{fu2024droid}, and Bridge \citep{ebert2022bridge}. Human raters ($N{=}5$) scored rationale quality on a scale of 1--5; we report the percentage preferring our rationales over the baseline.

\section{Additional Analysis}
\label{app:extra}
We include additional ablations. Figure~\ref{fig:extra} reports the effect of the number of future tokens on accuracy and interpretability, and the sensitivity to the EMA rate.

\begin{figure}[t]
  \centering
  \includegraphics[width=\columnwidth]{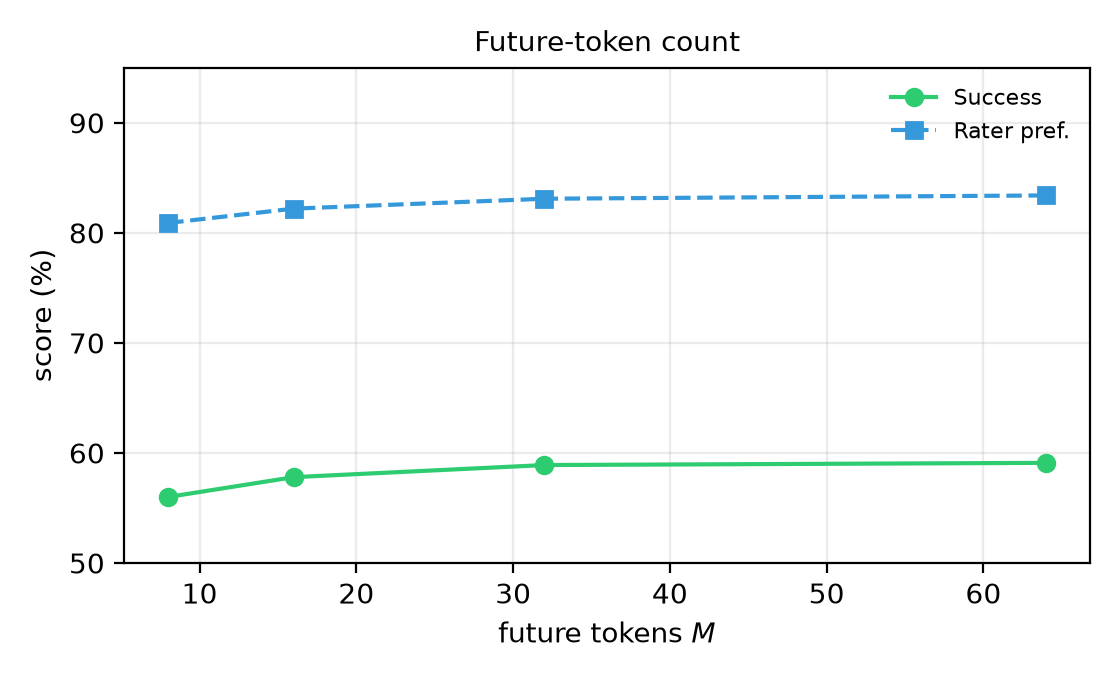}
  \caption{\textbf{Additional analysis.} Effect of future-token count and EMA rate on success and interpretability.}
  \label{fig:extra}
\end{figure}

\section{Prompt Templates}
\label{app:prompt}
AWM-VLA is trained with language-conditioned imitation \citep{agarwal2022language} and requires only natural-language instructions (e.g., ``pick up the cup and place it on the plate''). The object-centric rationale is generated by decoding the future tokens and phrasing the top-scoring object change as a short sentence, following the multi-instruction decoupling spirit of \citep{yang2026instrucrobo}. We found the exact phrasing matters little; the object-level prediction does the heavy lifting.

\end{document}